\documentclass[11pt]{article}

\usepackage[margin = 1in]{geometry}

\usepackage[utf8]{inputenc} 
\usepackage[T1]{fontenc}    
\usepackage{hyperref}       
\usepackage{url}            
\usepackage{booktabs}       
\usepackage{amsfonts}       
\usepackage{amsmath,amsthm,amssymb,mathtools} 
\usepackage{nicefrac}       
\usepackage{microtype}      
\usepackage{lipsum}
\usepackage{fancyhdr}       
\usepackage{graphicx}       
\usepackage{multirow}
\graphicspath{{media/}}     

\usepackage{authblk}

\usepackage{natbib}
\setcitestyle{numbers,square}
\usepackage{subfigure}
\usepackage{enumitem}
\setlist[itemize]{leftmargin=1cm}
\setlist[enumerate]{leftmargin=1cm}
\usepackage{enumitem}

\usepackage{algorithm}
\usepackage{algorithmic}

\usepackage{multirow}
\usepackage{bbm}

\usepackage{xcolor}

\newtheorem{theorem}{Theorem}[section]

\newtheorem{lemma}[theorem]{Lemma}

\theoremstyle{definition}

\newtheorem{assumption}[theorem]{Assumption}
\theoremstyle{remark}
\newtheorem{remark}[theorem]{Remark}

\title{MissDiff: Training Diffusion Models on Tabular Data \\ with Missing Values
}

\author[1]{Yidong~Ouyang}
\author[1]{Liyan~Xie}
\author[2]{Chongxuan~Li}
\author[3]{Guang~Cheng\thanks{Email: guangcheng@ucla.edu.}}

\affil[1]{\small
School of Data Science, The Chinese University of Hong Kong, Shenzhen}

\affil[2]{\small
Gaoling School of AI, Renmin University of China}

\affil[3]{\small Department of Statistics,
  University of California, Los Angeles}

\begin{document}

\date{\vspace{-20pt}}

\maketitle

\begin{abstract}
The diffusion model has shown remarkable performance in modeling data distributions and synthesizing data. However, the vanilla diffusion model requires complete or fully observed data for training. Incomplete data is a common issue in various real-world applications, including healthcare and finance, particularly when dealing with tabular datasets. This work presents a unified and principled diffusion-based framework for learning from data with missing values under various missing mechanisms. We first observe that the widely adopted ``impute-then-generate'' pipeline may lead to a biased learning objective. Then we propose to mask the regression loss of Denoising Score Matching in the training phase. We prove the proposed method is consistent in learning the score of data distributions, and the proposed training objective serves as an upper bound for the negative likelihood in certain cases. 
The proposed framework is evaluated on multiple tabular datasets using realistic and efficacious metrics and is demonstrated to outperform state-of-the-art diffusion model on tabular data with ``impute-then-generate'' pipeline by a large margin.

\end{abstract}

\section{Introduction}\label{sec:intro}

Diffusion models has emerged as an effective tool for modeling the data distribution and synthesize various type of data, such as images \citep{Ho2020DDPM,Song2021ScoreBasedGM, Dhariwal2021DiffusionMB, Rombach2021HighResolutionIS}, videos \citep{Ho2022VideoDM}, point clouds \citep{Luo2021DiffusionPM}, and tabular data \citep{Kim2023STaSyST, Kotelnikov2022TabDDPMMT}. It is known that such machine learning models typically rely on high-quality training data, which are usually expected to be free of missing values. In reality, it is often challenging to obtain complete data, particularly in healthcare, finance, recommendation systems, and social networks, due to privacy concerns, high cost or sampling difficulties, and the skewed distribution of user-generated content. 
For example, the respiratory rate of a patient may not have been measured, either because it was deemed unnecessary or was accidentally not recorded \citep{Yoon2017DiscoveryAC,Alaa2016PersonalizedRS,Yoon2018GAINMD}. Additionally, some information may be difficult or even dangerous to acquire, such as information obtained through a biopsy, which may not have been gathered for some reasons \citep{Yoon2018PersonalizedSP}.



In this work, we focus on learning a generative model from training data containing a significant amount of missing values, a problem that has been largely overlooked in the literature despite its widespread practical applications. Deep generative models, particularly diffusion modeimagels, can be used to augment training data and enhance the performance of image classification tasks \citep{Azizi2023SyntheticDF,You2023DiffusionMA} and adversarial robustness \citep{gowal2021improving, sehwag2021improving, Ouyang2022ImprovingAR}. Following this idea, we can achieve better performance for downstream tasks by utilizing generative model learning on incomplete data for synthetic data generation.
We will primarily utilize {\it tabular} data as examples, especially for the numerical experiments, as tabular data is a commonly encountered data type and frequently contains missing values in various applications \cite{Yoon2017DiscoveryAC,Alaa2016PersonalizedRS}. Moreover, by considering tabular data as an example, we will simultaneously study the missing value scenarios in  categorical and continuous variables, which are both contained in tabular type data. 

To deal with missing values in the training data, numerous studies propose to use various imputation methods and then train the model on the imputed data. Taking tabular data as an example, some approaches involve deleting instances (rows) or features (columns) with missing data or replacing missing values with the mean of observed values for that feature (column). Other methods employ machine learning approaches \citep{Buuren2011MICEMI,Bertsimas2017FromPM} or deep generative models for imputation tasks \citep{Yoon2018GAINMD,Biessmann2019DataWigMV,Wang2020PCGAINPC,Ipsen2020HowTD,Muzellec2020MissingDI}. It has been shown that imputation may reduce the diversity of the training data and may lead to biased performances in downstream tasks \citep{Bertsimas2021PredictionWM, Ipsen2020HowTD}.

In addition to imputation or simple deletion methods, previous work also studied learning from data with missing values and synthesizing complete data using GAN or VAE architectures \citep{Li2019MisGANLF, Li2020LearningFI,Neves2022FromMD}. Compared with our proposed framework, these methods involve training additional networks, impose certain assumptions on the missing mechanisms, and the unique challenges associated with tabular data are less investigated.

In this work, we propose a diffusion-based framework, which we call {\it MissDiff}, for learning from data with missing values. We present the theoretical justifications of {\it MissDiff} on recovering the oracle score function and upper bounding the negative likelihood on the data under mild assumptions on the missing mechanisms. To the best of our knowledge, this is the first work that learns a generative model from mixed-type data containing missing values, and the missing values are used directly in the training process without prior imputation. 
We conduct a suite of numerical experiments on mixed-type tabular data, comprising both continuous and categorical variables, under various missing mechanisms. Evaluated under several realistic and efficacious metrics, {\it MissDiff} consistently outperforms other baseline methods by a considerable margin.

Our contributions can be summarized as follows. \begin{itemize}
    \item  We propose a diffusion-based framework, which we call {\it MissDiff}, for generative model training from incomplete data. We mainly focus on tabular data generation, which contains both categorical and continuous variables. 
    \item We provide the theoretical justifications of MissDiff on recovering the oracle score function and upper bounding the negative likelihood on the data under mild assumptions on the missing mechanisms. 
    \item We conduct extensive numerical experiments on multiple real tabular datasets under different missing mechanisms to demonstrate the effectiveness of {\it MissDiff}.
\end{itemize}

The rest of the paper is organized as follows. Section \ref{sec:prelim} reviews the setup of the missing data mechanism and the diffusion model. Section \ref{sec:method} introduces the proposed method. Section \ref{sec:theory} theoretically characterizes the effectiveness of the proposed method. Numerical results are given in Section \ref{sec:numerical}. We conclude the paper in Section \ref{sec:conclusion}. All proofs and additional numerical examples and details are deferred to the appendix.

\subsection{Related Work}\label{sec:related}

One line of research focuses on different types of generative models trained directly on data with missing values. These studies carefully modify the architecture and training objectives of Generative Adversarial Network (GAN) or Variational Autoencoder (VAE) to learn from incomplete data \citep{Li2019MisGANLF,Li2020LearningFI,Ipsen2020HowTD}.


Another research direction explores learning generative models for imputing missing values in observed data \citep{Yoon2018GAINMD, Neves2022FromMD, Ipsen2020notMIWAEDG, Muzellec2020MissingDI, Tashiro2021CSDICS, Nazbal2018HandlingIH, Ma2020VAEMAD, Mattei2019MIWAEDG,Valera2017GeneralLF}. For example,  \cite{Tashiro2021CSDICS} proposes the conditional score-based generative model for time series imputation. Moreover, \citep{Zheng2022DiffusionMF} adapt the conditional score-based diffusion model proposed in \cite{Tashiro2021CSDICS} for imputing tabular data. Imputation methods cannot easily used for generating new complete data, which is the main difference with the first line of works.


Tabular data, as a mixed-type data that typically contains both categorical and continuous variables, has attracted significant attention in the field of machine learning. Tabular data synthesis has been a long-standing research topic in this area. The presence of mixed variable types and class imbalance for discrete variables make it a challenging task to model tabular data. Recently, several deep learning-based models have been proposed for generating tabular data \citep{Xu2019ModelingTD, Choi2017GeneratingMD, Srivastava2017VEEGANRM, Park2018DataSB, Kim2021OCTGANNO, Finlay2020HowTT, Kim2023STaSyST, Kotelnikov2022TabDDPMMT}. Among these methods, \citep{Kotelnikov2022TabDDPMMT} employs Gaussian transitions for continuous variables and multinomial transitions for discrete random variables, while \citep{Kim2023STaSyST} proposes a self-paced learning technique and a fine-tuning strategy for score-based models and achieves state-of-the-art performance in tabular data generation. Moreover, the discrete Score Matching methods proposed in \cite{meng2022concrete} and \cite{sun2023scorebased} can also be employed to handle discrete variables in tabular data.

\section{Problem Setup and Preliminaries}\label{sec:prelim}

\subsection{Training with Missing Data}

We aim to learn a diffusion-based generative model from training data that may contain a certain proportion of missing values. 
Following the settings in \cite{Little1988StatisticalAW, Li2019MisGANLF, Ipsen2020HowTD}, we denote the underlying complete $d$-dimensional data as $\mathbf{x}=(x_{1},\ldots,x_{d}) \in \mathcal{X}$ and assume it is sampled from the unknown true data-generating distribution $p_0(\mathbf{x})$. Here, each variable $x_i$, $i=1,\ldots,d$, can be either categorical or continuous. For each data point $\mathbf{x}$, suppose there is a binary mask $\mathbf{m}=(m_{1},\ldots,m_{d}) \in\{0,1\}^{d}$ which indicates the missing entry for the current sample, i.e.,
\[
m_{i}= 
\begin{cases} 
1 & \text { if } x_i \text { is observed, } \\ 
0 & \text { if } x_i \text { is missing. }
\end{cases}
\]
Then, the observed data $\mathbf{x}^{\text {obs}}= \mathbf{x}\odot \mathbf{m}+ \mathrm{na} \odot (\mathbf{1}-\mathbf{m})$, where $\mathrm{na}$ indicates the missing value, $\odot$ denotes element-wise multiplication, and $\mathbf{1}$ is the all-one vector.

Suppose we have $n$ complete (unobservable) training data points $\mathbf{x}_1,\ldots,\mathbf{x}_n \overset{iid}{\sim} p_0(\mathbf{x})$ and simultaneously $n$ corresponding masks $\mathbf{m}_1,\ldots,\mathbf{m}_n$ generated from a specific missing data mechanism detailed later. Then the observed data values are $S^{\text {obs}}=\{\mathbf{x}_i^{\text {obs}}\}_{i=1}^{n}$ with $\mathbf{x}_i^{\text {obs}}= \mathbf{x}_i\odot \mathbf{m}_i+ \mathrm{na} \odot (\mathbf{1}-\mathbf{m}_i)$. 
The missing mechanisms can be categorized based on the relationships between the mask $\mathbf{m}$ and the complete data $\mathbf{x}$  \citep{Little1988StatisticalAW} as follows,
\begin{itemize}
    \item  Missing Completely At Random (MCAR): mask $\mathbf{m}$ is independent with the completed data $\mathbf{x}$.
    \item Missing At Random (MAR): mask $\mathbf{m}$ only depends on the observed value $\mathbf{x}^{\text{obs}}$.
    \item Not Missing At Random (NMAR): $\mathbf{m}$ depends on the observed value $\mathbf{x}^{\text{obs}}$ and missing value.
\end{itemize}

Compared with previous work which typically develop their algorithms and theoretical foundations under the MCAR assumption \cite{Li2019MisGANLF,Ipsen2020HowTD, Yoon2018GAINMD,Li2020LearningFI}, our method and theoretical guarantees aim to provide a general framework for learning on incomplete data and generate complete data.

Our objective is to train a generative model $p_{\mathbf{\phi}}$, parametrized by the neural network parameters $\phi$, using the observed data $S^{\text {obs}}$, such that $p_{\mathbf{\phi}}$ is close to the true distribution $p_0(\mathbf{x})$ and we can efficiently generate synthetic data from $p_{\mathbf{\phi}}$. The optimal model is the one that maximizes the likelihood of the data sampling from the true data-generating distribution $p(\mathbf{x})$, which is solved via maximizing the likelihood function on training samples $\prod_i p_\phi(\mathbf{x}_i)$ and thus it logarithm $\sum_i \log p_\phi(\mathbf{x}_i)$, which corresponds to the following objective in the population sense
\[
\mathbf{\phi} = \underset{\mathbf{\phi}}{\arg\max} \mathbb{E}_{\mathbf{x}\sim p(\mathbf{x})}[\log p_{\mathbf{\phi}}(\mathbf{x})],
\] 
or equivalently, the model that minimizes the Kullback-Leibler (KL) divergence between the true data-generating distribution $p(\mathbf{x})$ and the model $p_{\mathbf{\phi}}(\mathbf{x})$. In the following, we mainly consider the score-based generative model as $p_{\mathbf{\phi}}$.



\subsection{Score-Based Generative Model}\label{sub:VDM}

In this work, we adopt the diffusion model as the prototype for developing our proposed method. We propose to train the model with missing values directly without the need for prior imputation. We first briefly review the key components of score-based generative models \cite{Ho2020DDPM, Song2021ScoreBasedGM}.




Score-based generative models are a class of generative models that learn the score function, which is the gradient of the log-density of the data distribution. These models have gained attention due to their flexibility and effectiveness in capturing complex data distributions. Following the notation in \cite{Song2021ScoreBasedGM}, the score-based generative models are based on a forward stochastic differential equation (SDE), $\mathbf{x}(t)$ with $t\in[0,T]$, defined as 
\begin{equation}\label{eq:forward}
    \mathrm{d} \mathbf{x}(t)=\mathbf{f}(\mathbf{x}(t), t) \mathrm{d} t+g(t) \mathrm{d} \mathbf{w},
\end{equation}
where $\mathbf{w}$ is the standard Wiener process (Brownian motion), $\mathbf{f}(\cdot, t): \mathbb{R}^d \rightarrow \mathbb{R}^d$ is a vector-valued function called the drift coefficient of $\mathbf{x}(t)$, and $g(\cdot): \mathbb{R} \rightarrow \mathbb{R}$ is a scalar function known as the diffusion coefficient of $\mathbf{x}(t)$. 

The solution of a stochastic differential equation is a continuous trajectory of random variables $\{\mathbf{x}(t)\}_{t\in[0,T]}$. Let $p(\mathbf{x})$ denote the path measure for the trajectory $\mathbf{x}$ on $[0,T]$, $p_t(\mathbf{x})$ denote the marginal probability density function of $\mathbf{x}(t)$, and $p(\mathbf{x}(t)|\mathbf{x}(s))$ denote the conditional probability density of $\mathbf{x}(t)$ conditioned on $\mathbf{x}(s)$, where $s<t$ is a previous time point. When constructing the SDE, we let $p_0(\mathbf{x})$ be the true data distribution, and after perturbing the data according to the SDE, the data distribution becomes $p_T(\mathbf{x})$ which is close to a tractable noise distribution, usually set as the standard Gaussian distribution.

The data generation process is performed via the reverse SDE, i.e., first sampling data $\mathbf{x}_T$ from $p_T(\mathbf{x})$ and then generate $\mathbf{x}_0$ through the reverse of \eqref{eq:forward}. For any SDE in \eqref{eq:forward}, the corresponding backward/reverse process is 
\begin{equation}\label{eq:backward}
\mathrm{d} \mathbf{x}(t)=\left[\mathbf{f}(\mathbf{x}(t), t)-g(t)^2 \nabla_{\mathbf{x}} \log p_t(\mathbf{x})\right] \mathrm{d} t+g(t) \mathrm{d} \overline{\mathbf{w}},
\end{equation}
where $\overline{\mathbf{w}}$ is a standard Wiener process when time flows backwards from $T$ to 0, and $\mathrm{d}t$ is an infinitesimal negative timestep. 

We can generate new data by running backward the reverse-time SDE \eqref{eq:backward} when the score of each marginal distribution, $\nabla_{\mathbf{x}} \log p_t(\mathbf{x})$ is known. Score Matching \citep{Hyvrinen2005EstimationON,Vincent2011ACB,Song2019SlicedSM} can be used for training a score-based model $\mathbf{s}_{\boldsymbol{\theta}}(\mathbf{x}(t), t)$ to estimate the score:
\begin{equation}\label{eq:dsm}
\boldsymbol{\theta}^*=\underset{\boldsymbol{\theta}}{\arg \min } \mathbb{E}_t\left\{\lambda(t) \mathbb{E}_{p(\mathbf{x}(0))} \mathbb{E}_{\mathbf{x}(t) \mid \mathbf{x}(0)}\left[\left\|\mathbf{s}_{\boldsymbol{\theta}}(\mathbf{x}(t), t)-\nabla_{\mathbf{x}(t)} \log p(\mathbf{x}(t) | \mathbf{x}(0))\right\|_2^2\right]\right\},
\end{equation}
where $\lambda:[0, T] \rightarrow \mathbb{R}_{>0}$ is a positive weighting function, $t$ is uniformly sampled over $[0, T]$, $\mathbf{x}(0) \sim p_0(\mathbf{x})$ and $\mathbf{x}(t) \sim p(\mathbf{x}(t) | \mathbf{x}(0))$. The local consistency of score mathcing is shown in \citep{Hyvrinen2005EstimationON}, i.e., $\mathbb{E}_{p(\mathbf{x}(0))} [\left\|\mathbf{s}_{\boldsymbol{\theta}}(\mathbf{x})-\nabla_{\mathbf{x}} \log p(\mathbf{x})\right\|_2^2]=0 \Leftrightarrow \boldsymbol{\theta}=\boldsymbol{\theta}^*$ under the assumption that there exists an unque $\boldsymbol{\theta}^*$ such that the true score function $\nabla_{\mathbf{x}} \log p(\mathbf{x})$ can be represented by $s_{\boldsymbol{\theta}^*}$.
\citep{Vincent2011ACB} builds the connection between Denoising Score Matching and Score Matching, and \citep{Song2019SlicedSM} further proves Sliced Score Matching can learn the consistent estimator of the oracle score and the asymptotic normality for the Sliced Score Matching.

\section{Method}\label{sec:method}



In general, there are two approaches to learn a generative model from incomplete data. The first approach is to construct a complete training data set first and then learn a generative model on the complete data. We can either delete instances (rows) or features (columns) with missing data or adopt ``inpute-then-generate'' paradigm, i.e., we complete the data either by traditional imputation methods or training machine learning imputation models \citep{Buuren2011MICEMI, Bertsimas2017FromPM} or deep generative models for imputation tasks \citep{Vincent2008ExtractingAC,Yoon2018GAINMD, Biessmann2019DataWigMV, Wang2020PCGAINPC, Ipsen2020HowTD, Muzellec2020MissingDI}. However, this pipeline may bring bias to the training objective. We clarify this claim in remark \ref{remark:challenge}. Moreover, some of the imputation methods require the ground truth for the missing values and increase the computational costs for learning additional networks, which highlights the need for alternative unbiased approaches that can handle missing data directly and more effectively.

The second way is to learn the generative model to deal with the incomplete data directly. In \cite{Li2019MisGANLF,Li2020LearningFI,Ipsen2020HowTD}, learning from incomplete data with GAN or VAE is proposed, which require training some additional networks and have assumptions on the missing mechanism. Moreover, the case of tabular data that contains mixed-type variables was not fully considered in those previous work. Therefore, a general framework for learning on missing data and generating complete data on tabular data is needed.

The diffusion model mentioned in Section \ref{sub:VDM} typically requires complete data for training, which is not practical for various real-world applications. To address this limitation, we propose a diffusion-based framework designed for training diffusion models on tabular data with missing values, making it more suitable for real-world scenarios where incomplete data is common.






\vspace{2pt}
\begin{remark}[Challenges when training generative models with missing data] \label{remark:challenge}
    Inspired by the analysis pipeline of ``inpute-then-regress'' \citep{Bertsimas2021PredictionWM, Ipsen2020HowTD} for the prediction task, we can study a corresponding framework for the generation task. The generative model $p_{\mathbf{\phi}}$ represents the probability distribution of the synthetic data $\mathbf{x}$. Under the maximum likelihood framework, $\mathbf{\phi}^* = \arg \max_{\mathbf{\phi}} \mathbb{E}_{\mathbf{x}\sim p_0(\mathbf{x})}[\log p_{\mathbf{\phi}}(\mathbf{x})]$. When data has missing values, the general approach, known as ``impute-then-generate'', may be used in practice. In this approach, the observed data $\mathbf{x}^{\text{obs}}$ is first imputed using an imputation model $f_{\mathbf{\varphi}}$. Then, the generative model is trained by maximizing the likelihood of imputed data, i.e., $\max_{\phi} \log p_{\mathbf{\phi}}(\mathbf{x}^{\mathrm{obs}}, \mathbf{x}^{\mathrm{miss}}:= f_{\mathbf{\varphi}}( \mathbf{x}^{\mathrm{obs}}))$ and with the special case of $f_{\mathbf{\varphi}}$ being the mean imputation, we have 
    $\max_{\phi} \log p_{\mathbf{\phi}}(\mathbf{x}^{\mathrm{obs}}, \mathbb{E}_{p_{\mathbf{\phi}}\left(\mathbf{x}^{\mathrm{miss}} \mid \mathbf{x}^{\mathrm{obs}}\right)}\left[\mathbf{x}^{\mathrm{miss}}\right])$. 
    However, this pipeline is biased because with single imputation, the conditional distribution over the missing data is discarded, and the optimal single imputation can no longer capture the data variability. 
\end{remark}
\vspace{2pt}

Motivated by the above observation, we instead train the model parameters $\mathbf{\phi}$ by maximizing the likelihood of $p_{\mathbf{\phi}}(\mathbf{x}^{\text {obs}}, \mathbf{x}^{\text {miss}})$. In general, $p_{\mathbf{\phi}}(\mathbf{x}^{\text {obs}}, \mathbf{x}^{\text {miss}})\neq p_{\mathbf{\phi}}(\mathbf{x}^{\mathrm{obs}}, \mathbb{E}_{p_{\mathbf{\phi}}\left(\mathbf{x}^{\mathrm{miss}}| \mathbf{x}^{\mathrm{obs}}\right)}[\mathbf{x}^{\mathrm{miss}}])$. Therefore, we propose {\it MissDiff}, a diffusion-based framework for learning on missing data. This approach aims to address the limitations of traditional impute-then-generate methods by incorporating the uncertainty in missing data directly into the learning process, resulting in a more accurate generative model that better reflects the true data distribution.

We propose the following Denoising Score Matching method for data with missing values. 
Instead of using Eq \eqref{eq:dsm} for learning the score-based model $\mathbf{s}_{\boldsymbol{\theta}}(\mathbf{x}(t), t)$, we propose {\it MissDiff} as solution to
\begin{align}\label{eq:obj_mask}
\boldsymbol{\theta}^*=\underset{\boldsymbol{\theta}}{\arg \min }  J_{DSM}(\boldsymbol{\theta}) :=& \frac{T}{2}\mathbb{E}_t\Big\{\lambda(t) \mathbb{E}_{\mathbf{x}^{\text{obs}}(0)}  \mathbb{E}_{\mathbf{x}^{\text{obs}}(t) \mid \mathbf{x}^{\text{obs}}(0)}  \notag\\
&\left[\left\|\left(\mathbf{s}_{\boldsymbol{\theta}}(\mathbf{x}^{\text{obs}}(t), t)-\nabla_{\mathbf{x}^{\text{obs}}(t)} \log p(\mathbf{x}^{\text{obs}}(t) \mid \mathbf{x}^{\text{obs}}(0))\right) \odot \mathbf{m}\right\|_2^2\right]\Big\},
\end{align}
where $\lambda(t)$ is a positive weighting function, $\mathbf{m}=\mathbb{I}\{\mathbf{x}^{\text{obs}}(0)=\mathrm{na}\}$ indicated the missing entries in $\mathbf{x}^{\text{obs}}$ and $p(\mathbf{x}^{\text{obs}}(t)|\mathbf{x}^{\text{obs}}(0)) = \mathcal{N}(\mathbf{x}^{\text{obs}}(t);\mathbf{x}^{\text{obs}}(0), \beta_t \mathbb{I} )$ is the Gaussian transition kernel. 
To make the transition $p(\mathbf{x}^{\text{obs}}(t)|\mathbf{x}^{\text{obs}}(0))$ and the gradient $\nabla_{\mathbf{x}^{\text{obs}}(t)} \log p(\mathbf{x}^{\text{obs}}(t) \mid \mathbf{x}^{\text{obs}}(0))$ well defined for the mixed-type data, we use 0 to replace $\mathrm{na}$ for continuous variables and a new category to represent $\mathrm{na}$ for discrete variables, which is the same operation as in \cite{Nazbal2018HandlingIH,Ma2020VAEMAD}. One-hot embedding is applied to discrete variables. More implementation details can be found in Appendix \ref{ap:training}.



More specifically, we mainly adopt the Variance Preserving (VP) SDE in this paper although Variance Exploding (VE) SDE \citep{Song2021ScoreBasedGM} is also appliable. The forward diffusion process of the Variance Preserving SDE is defined as 
\[
\mathrm{d} \mathbf{x}=-\frac{1}{2} \beta(t) \mathbf{x} \mathrm{d} t+\sqrt{\beta(t)} \mathrm{d} \mathbf{w}, 
\]
where $\left\{\beta_t \in(0,1)\right\}_{t\in(0,T)}$ is the increasing sequence denoting the variance schedule. Algorithm \ref{alg} and Algorithm \ref{alg1} demonstrate the Denoising Score Matching objective on missing data and sampling procedure of {\it MissDiff}. We write $\mathbf{x}(t)$ as $\mathbf{x}_t$ in the algorithm box for simplicity. 


\begin{minipage}{0.52\textwidth}
\begin{algorithm}[H]
    \centering
    \caption{{\it MissDiff}: Denoising Score Matching on Data with Missing Values}\label{alg}
    \footnotesize
    \begin{algorithmic}[1]
      \REQUIRE Diffusion process hyperparameter $\beta_t$, denote $\alpha_t=1-\beta_t$ and $\bar{\alpha}_t=\prod_{s=1}^t \alpha_s$.
		\REPEAT
  \STATE Sample $\mathbf{x}^{\text{obs}}_0$ according to the data distribution and missing mechanism;
  \STATE Infer mask $\mathbf{m}=\mathbf{1}[\mathbf{x}^{\text{obs}}(0)=\mathrm{na}]$;
\STATE $t \sim \operatorname{Uniform}(\{1, \ldots, T\})$;
\STATE $\boldsymbol{\epsilon} \sim \mathcal{N}(\mathbf{0}, \mathbf{I})$;
\STATE Take gradient descent step on
$$
\nabla_{\boldsymbol{\theta}}\left\|(\boldsymbol{\epsilon}-\mathbf{s}_{\boldsymbol{\theta}}(\sqrt{\bar{\alpha}_t} \mathbf{x}_0^{\text{obs}} +\sqrt{1-\bar{\alpha}_t} \boldsymbol{\epsilon}, t))\odot \mathbf{m}\right\|^2.
$$
\UNTIL{converged.}

	\end{algorithmic}
 
\end{algorithm}
\end{minipage}
\hfill
\begin{minipage}{0.4\textwidth}
\begin{algorithm}[H]
    \centering
    \caption{Variance Preserving Sampling of {\it MissDiff}}\label{alg1}
    \footnotesize
\begin{algorithmic}[1]
      \REQUIRE Diffusion process hyperparameter $\beta_t$, denote $\alpha_t=1-\beta_t$ and $\bar{\alpha}_t=\prod_{s=1}^t \alpha_s$.
		\STATE Sample $\mathbf{x}_T \sim \mathcal{N}(\mathbf{0}, \mathbb{I})$;
		\STATE $t=T$;
    \WHILE{$t\neq 1$}
        \STATE Sample $\boldsymbol{\epsilon}_t \sim \mathcal{N}(\mathbf{0}, \mathbb{I})$;
        \STATE 
        $\mathbf{x}_{t-1}=\frac{1}{\sqrt{\alpha_t}}(\mathbf{x}_t-\frac{1-\alpha_t}{\sqrt{1-\overline{\bar{\alpha}}_t}} \mathbf{s}_{\boldsymbol{\theta}}(\mathbf{x}_t, t))+\sqrt{\beta_t} \mathbf{\epsilon}_t$;
        \STATE $t=t-1$;
    \ENDWHILE
    \RETURN $\mathbf{x}_0$.
	\end{algorithmic}
\end{algorithm}
\end{minipage}

\section{Theory}\label{sec:theory}

In this section, we examine the effectiveness of {\it MissDiff} by theoretically characterizing the Score Matching objective under mild conditions on the missing mechanisms and build a further connection between Score Matching and maximizing likelihood objective for training the diffusion model.

In the following theorem, we present our first theoretical result that verifies that Denoising Score Matching on missing data can learn the oracle score, i.e, the score on complete data. Theorem \ref{thm:dsm} states that the global optimal solution of Denoising Score Matching on missing data obtained by {\it MissDiff} is the same as the oracle score.
The proof can be found in Appendix \ref{proof:thm1}.

\begin{theorem}\label{thm:dsm}
Denote $\rho_i$, $i\in \{1,2, ..., d\}$ as the percentage of missing samples for the $i$-th entry in the training data. Suppose $\rho_{\text{max}}:=\max_{i=1,\ldots,d}\rho_i<1$. Let $\boldsymbol{\theta}^*$ be the solution to the training objective of {\it MissDiff} define in Eq \eqref{eq:obj_mask}. Then we have $$
\mathbf{s}_{\boldsymbol{\theta}^*}(\mathbf{x}(t), t)= \nabla_{\mathbf{x}(t)} \log p_0(\mathbf{x}(t)).
$$
\end{theorem}

It is well known that with careful design of the weighting function $\lambda_t$, Denoising Score Matching can upper bound the negative log-likelihood of the diffusion model on the complete data \citep{Song2021MaximumLT}. Therefore, it is straightforward to extend such connection to incomplete data scenarios, which is detailed in the following theorem. These results provide insightly connections between the training objective of {\it MissDiff} and the maximum likelihood objective of the generative model on observed data. 

\begin{theorem}\label{thm:ml_real}
The objective function of Denoising Score Matching on missing data is an upper bound for the negative likelihood of the generative model on observed data $\mathbf{x}^{\text{obs}}$ up to a constant, that is, for $\lambda_t=\beta_t$ and under the same condition of Theorem~\ref{thm:dsm} and mild regularity conditions detailed in Appendix \ref{ap:proof2}, we have
$$
-\mathbb{E}_{p(\mathbf{x}^{\text{obs}})}\left[\log p_{\mathbf{\theta}}(\mathbf{x})\right] \leq \frac{1}{1-\rho_{\text{max}}} J_{\mathrm{DSM}}\left(\mathbf{\theta} \right)+C_1,
$$
where $C_1$ is a constant independent of $\mathbf{\theta}$.
\end{theorem}

The proof of Theorem \ref{thm:ml_real} can be found in Appendix \ref{ap:proof2}. When there is missing value, Theorem \ref{thm:ml_real} shows that the Denoising score matching on incomplete data still upper bounds the likelihood of the incomplete data up to a constant coefficient $1/(1-\rho_{\text{max}})$.
When there is no data missing, $\mathbf{\rho}$ is all zero vector, then we have $1/(1-\rho_{\text{max}})=1$ and Theorem \ref{thm:ml_real} degenerates to the Corollary 1 in \cite{Song2021MaximumLT}, i.e., \[
-\mathbb{E}_{p(\mathbf{x})}[\log p_{\mathbf{\theta}}(\mathbf{x})] \leq J_{\mathrm{DSM}}(\boldsymbol{\theta} ; g(\cdot)^2)+C_1, 
\]where the $J_{\mathrm{DSM}}(\boldsymbol{\theta}; g(\cdot)^2)$ is the Denoising Score Matching on complete data.



\section{Experiments}\label{sec:numerical}
In this section, we demonstrate the effectiveness of the proposed method {\it MissDiff} using simulations and two real-world tabular datasets. We introduce the experimental setup, including datasets, baseline models, and evaluation criterion, in Section \ref{sec:setup}. The detailed experiment results under different designs of MCAR is presented in Section \ref{sec:exp}. The detailed experiment results under the missing mechanisms MAR and NMAR are shown in Section \ref{section:abl}.


\subsection{Experimental Setup}
\label{sec:setup}

\paragraph{Datasets}
We present a suite of numerical evaluations of the proposed {\it MissDiff} approach on a simulated Bayesian Network data, a real Census tabular dataset \citep{census}, and the MIMIC4ED tabular dataset \citep{Xie2022BenchmarkingED}, with various proportions of missing values. The primary goal is to compare the effectiveness in synthetic data generation of the proposed {\it MissDiff} with the alternative baselines detailed later in this subsection. 



The detailed description of the dataset can be found in Table \ref{tab:dataset}, which specifies the number of training data (\#Train), the number of testing data (\#Test), the number of categorical (discrete) variables in the tabular dataset (\#Categorical), and the number of continuous variables (\#Continuous). Moreover, the last column shows the evaluation task we adopted as detailed later. The details of the data generated from a Bayesian Network can be found in Appendix \ref{ap:BN}.

\begin{table*}[htbp]
  
  \centering
  \caption{Synethetic and Real-World Datasets Used in Experiments.}
    \resizebox{\textwidth}{!}{
    \begin{tabular}{l c c c c c c}
    \toprule
    Dataset & \#Train & \#Test & \#Categorical & \#Continuous & Utility  \\
     \midrule
    Bayesian Network & 2000     & 20000    & 3 & 2 & Multi-class classification  \\
    Census \citep{census} & 16000     & 4000    & 9 & 6 & Binary classification  \\
    MIMIC4ED \citep{Xie2022BenchmarkingED} & 353150     &   88287   & 46 & 27 & Regression \\
    
    \bottomrule
    \end{tabular}}
  \label{tab:dataset}%
\end{table*}%

\paragraph{Choice of Masks under Different Missing Mechanisms}
\label{par:mismech}
To evaluate the performance of {\it MissDiff} on different missing mechanisms, we give a detailed explanation of the practical implementatin of MCAR \citep{Li2019MisGANLF,Ipsen2020HowTD, Yoon2018GAINMD,Li2020LearningFI}, MAR, and NMAR \citep{Muzellec2020MissingDI}. 

\begin{itemize}
\item MCAR: there are three types of missing mechanisms in MCAR.
\begin{itemize}
\item Row Missing. For a given missing ratio $\alpha\in(0,1)$, we have the number of elements missing in each row (i.e., for each sample $\mathbf{x}_i$) is $\lfloor d\alpha \rfloor$, where $\lfloor z \rfloor$ is the greatest integer less than $z$, and the location/index of the missing entries is randomly chosen according to the uniform distribution.

\item Column Missing. For a given missing ratio $\alpha$, we have the number of elements missing in each column (for each feature) is $\lfloor n\alpha \rfloor$, and the location/index of the missing entries is randomly chosen according to the uniform distribution.
\item Independent Missing. Each entry in the table is masked missing according to the realization of a Bernoulli random variable with parameter $\alpha$.
\end{itemize}
\item MAR: a fixed subset of variables that cannot have missing values is first sampled. Then, the remaining variables will have missing values according to a logistic model with random weights, which takes the non-missing variables as inputs. The outcome of this logistic model is re-scaled to attain a given missing ratio $\alpha$.
\item NMAR: the same pipeline as MAR with the inputs of the logistic model are masked by the MCAR mechanism. We refer to \cite{Muzellec2020MissingDI} for more detailed explanations.
\end{itemize}
\begin{remark}\label{lemma1}
Under the three missing mechanisms in MCAR,  with the missing ratio parameter set as $0<\alpha<1$, condition in Theorem \ref{thm:dsm} can be satisfied with probability at least $1-\delta$, where $\delta=\max\{(\frac{\alpha d -1}{d})^n d, \alpha ,\alpha^n d\}$ and it will be sufficiently small when $\alpha$ is small and $n$ is sufficiently large.
\end{remark}
Remark \ref{lemma1} gives the guarantee that {\it MissDiff} can recover the oracle score under MCAR with high probability. In the following tables in Sections \ref{sec:exp} and \ref{section:abl}, we adopt the missing ratio $\alpha=0.2$ and XGBoost for the downstream tasks with no specific clarification. More experimental results can be found in Appendix \ref{ap:exper}.


\paragraph{Baseline Methods} \label{sec:baseline}

We compare the proposed method with several baseline methods for synthetic data generation training on data with missing values. The baseline methods are described as follows:
\begin{enumerate}
\item {\it Diff-delete}: Learn a vanilla diffusion model after deleting rows containing missing values.
\item {\it Diff-mean}: Learn a vanilla diffusion model after imputing missing values using the mean value in that column.
\item STaSy \citep{Kim2023STaSyST} with the above two data completion methods. STaSy is the state-of-the-art diffusion model on tabular data, which outperforms MedGAN \citep{Choi2017GeneratingMD}, VEEGAN \citep{Srivastava2017VEEGANRM}, CTGAN \citep{Xu2019ModelingTD}, TVAE \citep{Xu2019ModelingTD}, TableGAN \citep{Park2018DataSB}, OCTGAN \citep{Kim2021OCTGANNO}, RNODE \citep{Finlay2020HowTT} by a large margin. 
\end{enumerate}
We use the variance-preserving SDE with the time duration $T = 100$ for Bayesian Network and Census dataset and $T = 150$ for MIMIC4ED dataset. We use the standard pre/post-processing of tabular data to deal with mixed-type data \citep{Kim2023STaSyST,Kotelnikov2022TabDDPMMT,Zheng2022DiffusionMF}. i.e., we use the min-max normalization for the continuous variables and reverse its scaler when generation. We use one-hot embedding for the discrete variables and use the rounding function after the softmax function when generation. We train the diffusion model for 250 epochs with batch size 64. For more details, please refer to Appendix \ref{ap:training}.


\paragraph{Evaluation Criterion}\label{sec:evaluation}
Following \cite{Xu2019ModelingTD, Kim2023STaSyST, Kotelnikov2022TabDDPMMT}, we use two types of criterion, {\it fidelity} and {\it utility}, to evaluate the quality of the synthetic data generated. To evaluate the {\it fidelity} of synthetic data compared with real data, we adopt a model-agnostic library, SDMetrics \citep{sdmetrics}. The result is a float number range from 0 to 100\%. The larger the score, the better the overall quality of synthetic data is.

To evaluate the {\it utility} of synthetic data, we follow the same pipeline of \cite{Kim2023STaSyST}, i.e., training various models, including Decision Tree, AdaBoost, Logistic/Linear Regression, MLP classifier/regressor, RandomForest, and XGBoost, on synthetic data, and validate the model on original training data, and test them with real test data. For classification tasks, we mainly use classification accuracy and also report AUROC, F1, and Weighted-F1 in Appendix \ref{ap:exper}. For regression tasks, we mainly use the Root Mean Squared Error (RMSE) and also report $R^2$ in the Appendix \ref{ap:exper}. All the experiments are obtained from 3 repetitions.

\subsection{Experiment Results}
\label{sec:exp}
\subsubsection{Simulation Study}


\textit{Q1: How does {\it MissDiff} perform on different missing ratios against the vanilla diffusion model learned on the data completed by two baseline methods mentioned in section \ref{sec:baseline}? }

Figure \ref{tab:BN} summarizes the SDMetrics score on the simulated Bayesian Network dataset example. With the same diffusion model architecture and the same training hyperparameter, {\it MissDiff} achieves consistently better results against the vanilla diffusion model deleting the incomplete row or using the mean value for imputation when the missing ratio varies from 0.1 to 0.9. Moreover, the advantage of {\it MissDiff} becomes more obvious for large missing ratios. These experimental results verify the motivation of {\it MissDiff} proposed in Remark \ref{remark:challenge} that the learning objective of impute-then-generate is biased. Directly learning on the missing data can significantly enhance the performance of the learned generative model.

\begin{figure}[th]
\centering
\subfigure[Row missing]{\includegraphics[width=.32\textwidth]{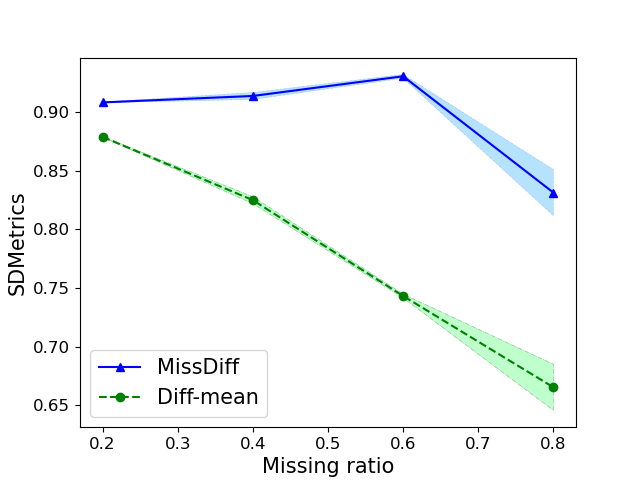}}
\subfigure[Column missing]{\includegraphics[width=.32\textwidth]{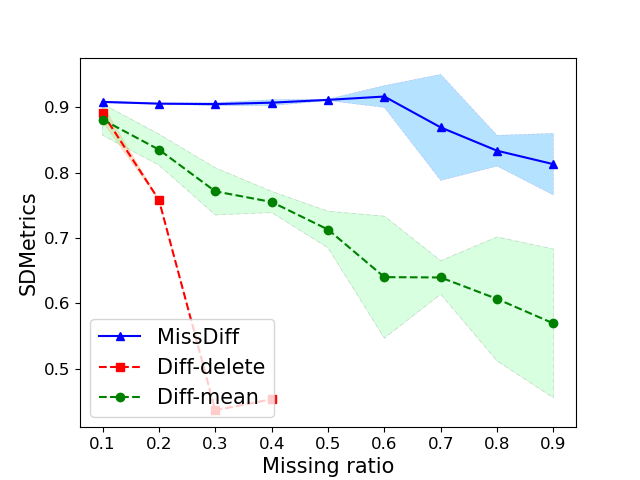}}
\subfigure[Independent missing]{\includegraphics[width=.32\textwidth]{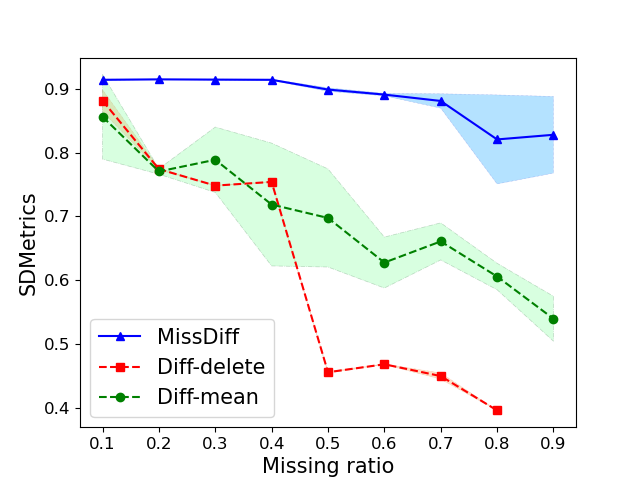}}
\caption{{\it Fidelity} evaluation of {\it MissDiff} on data generated by Bayesian Network under different missing ratios. We shade the area between mean $\pm$ std.
}
\label{tab:BN}
\end{figure}

\subsubsection{Real Tabular Datasets} \label{sec:expri}
\textit{Q2: How does {\it MissDiff} perform on more complicated real-world data and compared with state-of-the-art generative model on tabular data? }

Table \ref{tab:t1} and \ref{tab:t2} demonstrate the effectiveness of {\it MissDiff} on the Census dataset under MCAR. STaSy is a state-of-the-art generative model for tabular data, which means {\it MissDiff} achieves quite good performance on learning from incomplete data and generate complete data. More importantly, {\it MissDiff} achieves better performance than {\it STaSy-delete} and {\it STaSy-mean} even without adopting the self-paced learning technique and the fine-tuning strategy used by STaSy. Moreover, the results of {\it STaSy-delete} and {\it STaSy-mean} in Tables  \ref{tab:t1} and \ref{tab:t2} are obtained by training diffusion model for 1000 epochs, compared with 250 epochs of {\it MissDiff}, {\it Diff-delete}, and {\it Diff-mean}. If we reduce the training epochs of {\it STaSy-delete} and {\it STaSy-mean} to 250 epochs, the performance will degrade significantly, which can be found in Appendix \ref{ap:exper}. Compared with the state-of-the-art performance of STaSy for tabular data generation, its relatively worse performance here indicates that STaSy could be susceptible to missing data.
\begin{table}[htbp]
\centering
 \caption{{\it Fidelity} evaluation of {\it MissDiff} on Census dataset. ``-'' denotes the corresponding method cannot applied since no data $\mathbf{x}_i$ will be left after deleting the incomplete data. The {\it larger} the score, the {\it better} the overall quality of synthetic data is.}
\begin{tabular}{c|c|c|c|c|c}
\toprule
 & {\it MissDiff} & {\it Diff-delete}  & {\it Diff-mean} & {\it STaSy-delete}  & {\it STaSy-mean} \\
\midrule
Row Missing                         & \textbf{80.59}\%   &     -    & 76.92\%  & -& 56.75\% \\
Column Missing                     & \textbf{82.70}\%   & 75.03\% & 76.17\%  & 56.90\% & 51.54\% \\
Independent Missing                & \textbf{83.16}\%   & 74.94\% & 76.60\% & 56.07\% & 57.06\% \\

\bottomrule
\end{tabular}
\label{tab:t1}
\end{table}

\begin{table}[htbp]
\centering

 \caption{{\it Utility} (classification accuracy) evaluation of {\it MissDiff} on Census dataset. The {\it larger} the accuracy, the {\it better} the performance.}
\begin{tabular}{c|c|c|c|c|c}
\toprule   & {\it MissDiff} & {\it Diff-delete}  & {\it Diff-mean}  & {\it STaSy-delete}  & {\it STaSy-mean} \\
\midrule
Row Missing                        & \textbf{79.48}\%   &    -     & 78.45\%   & -  &  70.79\% \\
Column  Missing                    & 71.68\%   & 72.89\% & \textbf{79.60}\%  & 68.96\% & 74.47\%\\
Independent Missing                & \textbf{79.49}\%   & 75.39\% & 75.96\% & 78.36\% & 77.34\%\\ 

\bottomrule
\end{tabular}
\label{tab:t2}
\end{table}

\noindent\textit{Q3: How does {\it MissDiff} perform on real application of large-scale Electronic Health Records data?}

Table \ref{tab:t3} and \ref{tab:t4} show the performance of {\it MissDiff} on the MIMIC4ED dataset under MCAR. On this large dataset with dozens of continuous and discrete variables, {\it MissDiff} gives consistently better performance with the same training epochs (250 epochs).

\begin{table}[htbp]
\centering
 \caption{{\it Fidelity} evaluation of {\it MissDiff} on MIMIC4ED dataset.}
\begin{tabular}{c|c|c|c|c|c}
\toprule
 & {\it MissDiff} & {\it Diff-delete}  & {\it Diff-mean}  & {\it STaSy-delete}  & {\it STaSy-mean}  \\
\midrule
Row Missing                         & \textbf{84.45}\%    &  -    & 75.22\% &  -  &  82.94\% \\
Column Missing                     & \textbf{79.24}\%   &   -   & 76.57\% &   -   & 79.03\%\\
Independent Missing                & \textbf{78.01}\%   &   -   & 76.16\% &   -   & 77.21\% \\

\bottomrule
\end{tabular}
\label{tab:t3}
\end{table}

\begin{table}[ht!]
\centering
 \caption{{\it Utility} (RMSE) evaluation of {\it MissDiff} on MIMIC4ED dataset. The {\it lower} the RMSE, the {\it better} the performance.}
\begin{tabular}{c|c|c|c|c|c}
\toprule   & {\it MissDiff} & {\it Diff-delete}  & {\it Diff-mean}  & {\it STaSy-delete}  & {\it STaSy-mean}  \\
\midrule
Row Missing                          & \textbf{1.826}   &   -      & 2.166  &   -    & 1.894 \\
Rolumn Missing                     & \textbf{1.834}   & - & 2.011  &   -      &  1.935\\
Independent Missing                & \textbf{1.852}   & - & 2.483  &   -      &  1.972 \\ 

\bottomrule
\end{tabular}
\label{tab:t4}
\end{table}

\subsection{Ablation Study}
\label{section:abl}
\textit{Q4: How does {\it MissDiff} perform on other missing mechanisms beyond MCAR, i.e., MAR and NMAR?} 

Table \ref{tab:t5} and \ref{tab:t6} demonstrate the effectiveness of {\it MissDiff} on the Census dataset beyond MCAR. The results show the great potential of learning directly on the missing data when the missing mechanism is not MCAR, which cannot be easily dealt with by previous methods \citep{Li2019MisGANLF,Ipsen2020HowTD, Yoon2018GAINMD,Li2020LearningFI}.

\begin{table}[ht!]
\centering
 \caption{{\it Fidelity} evaluation of {\it MissDiff} on Census dataset under MAR, NMAR with missing ratio~0.2.}
\begin{tabular}{c|c|c|c}
\toprule   & {\it MissDiff} & {\it Diff-delete}  & {\it Diff-mean}    \\
\midrule
MAR                         & \textbf{77.45}\%	& 73.78\%	& 76.08\%   \\
NMAR                     & \textbf{77.88}\%    &	75.72\%    &	76.97\%   \\

\bottomrule
\end{tabular}
\label{tab:t5}
\end{table}

\begin{table}[ht!]
\centering
 \caption{{\it Utility} (classification accuracy) evaluation of {\it MissDiff} on Census dataset under MAR, NMAR.}
\begin{tabular}{c|c|c|c}
\toprule   & {\it MissDiff} & {\it Diff-delete}  & {\it Diff-mean}    \\
\midrule
MAR                         & \textbf{79.95}\%	& 69.475\%	& 77.425\%   \\
NMAR                     & \textbf{80.95}\%    &	66.5\%    &	80.025\%   \\

\bottomrule
\end{tabular}
\label{tab:t6}
\end{table}


\noindent \textit{Q5: How does {\it MissDiff} perform on imputation task?} 

Furthermore, we validate the performance of the propose {\it MissDiff} on imputation tasks for the Census dataset. We compare {\it MissDiff} with state-of-the-art imputation approaches in Table \ref{tab:imputation}. The result shows that although designed for generation tasks, {\it MissDiff} also performs well for imputation tasks.

\begin{table}[ht!]
\centering
\caption{Imputation result comparisons on the Census dataset. The {\it lower} the RMSE, the {\it better} the performance.}
 \resizebox{0.99\textwidth}{!}{
 \begin{tabular}{ccccccc}
\toprule
Method & Mean /Mode & MICE(linear)\citep{Buuren2011MICEMI} & MissForest \citep{Stekhoven2015missForestNM}& GAIN\citep{Yoon2018GAINMD} & CSDI$\_$T \citep{Zheng2022DiffusionMF} & {\it MissDiff}\\
\midrule
RMSE & 0.120 &0.101 &0.112 &0.123 &0.099&\textbf{0.087}\\
\bottomrule
\end{tabular}}
\label{tab:imputation}
\end{table}

\section{Conclusion and Discussion}\label{sec:conclusion}

We propose a diffusion-based generative framework, called {\it MissDiff}, for synthetic data generation trained on data with missing values directly. Compared with traditional methods of handling missing data, such as deletion or imputation, which may lead to reduced data diversity and biased performance, {\it MissDiff} offers a promising alternative that directly handles missing data without the need for imputation or deletion. Theoretical justification for {\it MissDiff}'s effectiveness is provided. Moreover, extensive numerical experiments demonstrate strong empirical evidence for the effectiveness of {\it MissDiff}. 

\paragraph{Limitations and broader impact}
Overall, this research presents a promising direction for handling missing data in generative model training. The proposed framework, {\it MissDiff}, has potential applications in a wide range of domains where missing data is a common issue. A potential limitation of this work is  that it has only been empirically validated on tabular data. For future directions, it would be interesting to see how {\it MissDiff} performs empirically with more complicated data types such as video or language data. Furthermore, further research could explore the theoretical effectiveness of {\it MissDiff} on the utility perspective or differential privacy perspective.

\section{Acknowledgement}
Yidong Ouyang is partially supported by the Shenzhen Research Institute of Big Data PhD Fellowship. Yidong Ouyang and Liyan Xie are partially supported by UDF01002142 through The Chinese University of Hong Kong, Shenzhen. Guang Cheng is partially supported by Office of Naval Research, ONR (N00014-22-1-2680), NSF – SCALE MoDL (2134209), and Meta gift fund.


\appendix

\section{Proofs for Section 4}

\subsection{Proof of Theorem \ref{thm:dsm}}
\label{proof:thm1}

In order to show Theorem \ref{thm:dsm}, we aim to show that the optimal solution $\boldsymbol{\theta}^*$, which minimizes the  objective function $J_{DSM}(\boldsymbol{\theta})$ satisfies $\mathbf{s}_{\boldsymbol{\theta}^*}(\mathbf{x}(t), t)= \nabla_{\mathbf{x}(t)} \log p_t(\mathbf{x}(t))$, i.e., the optimal solution to the population loss function can recover the oracle score function.

For the Gaussian transition distribution that we used with the isotropic covariance matrix, the score on the incomplete data is equivalent to the score on the complete data when performing element-wise multiplication with mask, i.e., $\nabla_{\mathbf{x}^{\text{obs}}(t)} \log p(\mathbf{x}^{\text{obs}}(t) |\mathbf{x}^{\text{obs}}(0)) \odot \mathbf{m} = \nabla_{\mathbf{x}(t)} \log p(\mathbf{x}(t) |\mathbf{x}(0)) \odot \mathbf{m}$\footnote{Assume $p(\mathbf{x}^{\text{obs}}(t) |\mathbf{x}^{\text{obs}}(0))=\mathcal{N}(\mathbf{x}^{\text{obs}}(t);\mathbf{\mu}^{\text{obs}},\Sigma)$ and $p(\mathbf{x}(t) |\mathbf{x}(0))=\mathcal{N}(\mathbf{x}(t);\mathbf{\mu},\Sigma)$, with $\Sigma=(1-\bar{\alpha}_t)\mathbb{I}$ and $\mathbf{\mu}^{\text{obs}} =\mathbf{\mu} \odot \mathbf{m}$. It is not hard to see $\nabla_{\mathbf{x}^{\text{obs}}(t)} \log p(\mathbf{x}^{\text{obs}}(t) |\mathbf{x}^{\text{obs}}(0)) \odot \mathbf{m} = -(\mathbf{x}^{\text{obs}}(t)-\mathbf{\mu}^{\text{obs}})\odot \mathbf{m} = -(\mathbf{x}(t)-\mathbf{\mu})\odot \mathbf{m} = \nabla_{\mathbf{x}(t)} \log p(\mathbf{x}(t) |\mathbf{x}(0)) \odot \mathbf{m}$. }, where $\mathbf{m}=\mathbbm{1}\{\mathbf{x}^{\text{obs}}(0)=\mathrm{na}\}$ indicated the missing entries in $\mathbf{x}^{\text{obs}}(0)$. Therefore, under certain conditions, we may first relate the Denosing Score Matching objective on missing data to the Denosing Score Matching objective on the complete data,
$$
\begin{aligned}
&\mathbb{E}_{p(\mathbf{x}^{\text{obs}}(0),\mathbf{m})}  \mathbb{E}_{p(\mathbf{x}^{\text{obs}}(t) | \mathbf{x}^{\text{obs}}(0))}  [\|(\mathbf{s}_{\boldsymbol{\theta}}(\mathbf{x}^{\text{obs}}(t), t)-\nabla_{\mathbf{x}^{\text{obs}}(t)} \log p(\mathbf{x}^{\text{obs}}(t) |\mathbf{x}^{\text{obs}}(0))) \odot \mathbf{m}\|_2^2] \\
& = \mathbb{E}_{p(\mathbf{x}(0),\mathbf{m})}  \mathbb{E}_{p(\mathbf{x}(t) | \mathbf{x}(0))}  [\|(\mathbf{s}_{\boldsymbol{\theta}}(\mathbf{x}(t), t)-\nabla_{\mathbf{x}(t)} \log p(\mathbf{x}(t) |\mathbf{x}(0))) \odot \mathbf{m}\|_2^2].
\end{aligned}
$$


Moreover, notice that we have 
$$
\begin{aligned}
&\mathbb{E}_{p(\mathbf{x}(0),\mathbf{m})}  \mathbb{E}_{p(\mathbf{x}(t) | \mathbf{x}(0))}   [\|(\mathbf{s}_{\boldsymbol{\theta}}(\mathbf{x}(t), t)-\nabla_{\mathbf{x}(t)} \log p(\mathbf{x}(t) |\mathbf{x}(0))) \odot \mathbf{m}\|_2^2] \\
&= \mathbb{E}_{p(\mathbf{x(0)},\mathbf{x(t)})}\| (\mathbf{s}_{\boldsymbol{\theta}}(\mathbf{x}(t), t)-\nabla_{\mathbf{x}(t)} \log p_t(\mathbf{x}(t) )) \odot \sqrt{\mathbb{E}_{p(\mathbf{m}|\mathbf{x}(0))}[\mathbf{m}]}\|_2^2],
\end{aligned}
$$
where $\sqrt{\boldsymbol{z}}$ denotes the element-wise operation on vector $\boldsymbol{z}$. The last equation is because we take the conditional expectation of the binary mask $\mathbf{m}$ and since $\mathbf{m}_i\in\{0,1\}$ we have $\mathbb{E}[\mathbf{m}_i^2]=\mathbb{E}[\mathbf{m}_i]$ for any distribution of $\mathbf{m}$. Assuming that $\mathbb{E}_{p(\mathbf{m}|\mathbf{x}(0))}[\mathbf{m}]\equiv \mathbf{1}-\boldsymbol\rho$ with $\boldsymbol\rho =[\rho_1,\ldots,\rho_d]$ and $\rho_i<1$, $i\in \{1,2, ..., d\}$ being the population percentage of missing samples for the $i$-th entry, we have $\mathbb{E}_{p(\mathbf{m}|\mathbf{x}(0))}[\mathbf{m}]>0$ and thus we can show the global optimal of Denoising Score Matching on missing data is the same as the oracle score.

\subsection{Proof of Theorem \ref{thm:ml_real}}\label{ap:proof2}

The notations are defined as follows. We let $\pi$ denote the pre-specified prior distribution (e.g., the standard normal distribution), $\mathcal{C}$ denote all continuous functions, and $\mathcal{C}^k$ denote the family of functions with continuous $k$-th order derivatives. Consider the MCAR missing mechanism. Denote $\rho_i$, $i\in \{1,2, ..., d\}$ as the population percentage of missing samples for the $i$-th entry in the training data. Suppose $\max_{i=1,\ldots,d}\rho_i<1$. In addition, we make the same mild regularity assumptions as \cite{Song2021MaximumLT} in the following. 
\begin{assumption} 
\begin{enumerate}[label=(\roman*)]
    \item $p(\mathbf{x}) \in \mathcal{C}^2$ and $\mathbb{E}_{\mathbf{x} \sim p_0}[\|\mathbf{x}\|_2^2]<\infty$.
\item $\pi(\mathbf{x}) \in \mathcal{C}^2$ and $\mathbb{E}_{\mathbf{x} \sim \pi}[\|\mathbf{x}\|_2^2]<\infty$.
\item $\forall t \in[0, T]: f(\cdot, t) \in \mathcal{C}^1, \exists C>0, \forall \mathbf{x} \in \mathbb{R}^d, t \in[0, T]:\|f(\mathbf{x}, t)\|_2 \leq C(1+\|\mathbf{x}\|_2)$.
\item $\exists C>0, \forall \mathbf{x}, \mathbf{y} \in \mathbb{R}^d:\|f(\mathbf{x}, t)-f(\mathbf{y}, t)\|_2 \leq C\|\mathbf{x}-\mathbf{y}\|_2$.
\item $g \in \mathcal{C}$ and $\forall t \in[0, T],|g(t)|>0$.
\item For any open bounded set $\mathcal{O}, \int_0^T \int_{\mathcal{O}}\|p_t(\mathbf{x})\|_2^2+d g(t)^2\|\nabla_{\mathbf{x}} p_t(\mathbf{x})\|_2^2 \mathrm{~d} \mathbf{x} \mathrm{d} t<\infty$.
\item $\exists C>0 \forall \mathbf{x} \in \mathbb{R}^d, t \in[0, T]:
\|\nabla_{\mathbf{x}} \log p_t(\mathbf{x})
\|_2 \leq C
(1+\|\mathbf{x}\|_2)$.
\item $\exists C>0, \forall \mathbf{x}, \mathbf{y} \in \mathbb{R}^d:\|\nabla_{\mathbf{x}} \log p_t(\mathbf{x})-\nabla_{\mathbf{y}} \log p_t(\mathbf{y})\|_2 \leq C\|\mathbf{x}-\mathbf{y}\|_2$.
\item $\exists C>0 \forall \mathbf{x} \in \mathbb{R}^d, t \in[0, T]:\|\mathbf{s}_{\boldsymbol{\theta}}(\mathbf{x}, t)\|_2 \leq C(1+\|\mathbf{x}\|_2)$.
\item $\exists C>0, \forall \mathbf{x}, \mathbf{y} \in \mathbb{R}^d:\|\mathbf{s}_{\boldsymbol{\theta}}(\mathbf{x}, t)-\mathbf{s}_{\boldsymbol{\theta}}(\mathbf{y}, t)\|_2 \leq C\|\mathbf{x}-\mathbf{y}\|_2$.
\item Novikov's condition: $\mathbb{E}[\exp (\frac{1}{2} \int_0^T\|\nabla_{\mathbf{x}} \log p_t(\mathbf{x})-\mathbf{s}_{\boldsymbol{\theta}}(\mathbf{x}, t)\|_2^2 \mathrm{~d} t)]<\infty$.
\item $\forall t \in[0, T], \exists k>0: p_t(\mathbf{x})=O(e^{-\|\mathbf{x}\|_2^k})$ as $\|\mathbf{x}\|_2 \rightarrow \infty$.
\end{enumerate}
\end{assumption} 

We mainly follow the proof strategy in \cite{Song2021MaximumLT}.
Consider the predefined SDE on the observed data, 
\begin{equation}\label{eq:5u}
    \mathrm{d} \mathbf{x}^{\text{obs}}=f(\mathbf{x}^{\text{obs}}, t) \mathrm{d} t+g(t) \mathrm{d} \mathbf{w},
\end{equation}
and the SDE parametrized by $\mathbf{\theta}$,
\begin{equation}
    \mathrm{d} \hat{\mathbf{x}}^{\text{obs}}_{\mathbf{\theta}}=\mathbf{s}_{\boldsymbol{\theta}}(\hat{\mathbf{x}}^{\text{obs}}_{\mathbf{\theta}}, t) \mathrm{d} t+g(t) \mathrm{d} \mathbf{w}.
\end{equation}

Let $\boldsymbol{\mu}$ and $\boldsymbol{\nu}$ denote the path measure of $\{\mathbf{x}^{\text{obs}}(t)\}_{t \in[0, T]}$ and $\{\hat{\mathbf{x}}^{\text{obs}}_{\mathbf{\theta}}(t)\}_{t \in[0, T]}$, respectively. Therefore, the distribution of $p_0(\mathbf{x})$ and $p_{\mathbf{\theta}}(\mathbf{x})$ can be represented by the Markov kernel $K(\{\mathbf{z}(t)\}_{t \in[0, T]}, \mathbf{y}):=\delta(\mathbf{z}(0)=\mathbf{y})$ as follow:
$$
\begin{gathered}
p_0(\mathbf{x}) = \int K(\{\mathbf{x}^{\text{obs}}(t)\}_{t \in[0, T]}, \mathbf{x}) \mathrm{d} \boldsymbol{\mu}(\{\mathbf{x}^{\text{obs}}(t)\}_{t \in[0, T]}),\\
p_{\mathbf{\theta}}(\mathbf{x})=\int K(\{\hat{\mathbf{x}}^{\text{obs}}_{\mathbf{\theta}}(t)\}_{t \in[0, T]}, \mathbf{x}) \mathrm{d} \boldsymbol{\nu}(\{\hat{\mathbf{x}}^{\text{obs}}_{\mathbf{\theta}}(t)\}_{t \in[0, T]}).
\end{gathered}
$$
According to the data processing inequality with this Markov kernel, the Kullback–Leibler (KL) divergence between the distribution of $p_0(\mathbf{x})$ and $p_{\mathbf{\theta}}(\mathbf{x})$ can be upper bounded, i.e., 
\begin{align}\label{eq:1u}
D_{\mathrm{KL}}(p_0 \| p_{\mathbf{\theta}}) 
= & D_{\mathrm{KL}}\Big(\int K(\{\mathbf{x}^{\text{obs}}(t)\}_{t \in[0, T]}, \mathbf{x}) \mathrm{d} \boldsymbol{\mu} \big\| \int K(\{\hat{\mathbf{x}}^{\text{obs}}_{\mathbf{\theta}}(t)\}_{t \in[0, T]}, \mathbf{x}) \mathrm{d} \boldsymbol{\nu}\Big) 
\leq  D_{\mathrm{KL}}(\boldsymbol{\mu} \| \boldsymbol{\nu}) .
\end{align}

By the chain rule of KL divergences,
\begin{equation}\label{eq:2u}
D_{\mathrm{KL}}(\boldsymbol{\mu} \| \boldsymbol{\nu})=D_{\mathrm{KL}}(p_T \| \pi)+\mathbb{E}_{\mathbf{z} \sim p_T}[D_{\mathrm{KL}}(\boldsymbol{\mu}(\cdot \mid \mathbf{x}^{\text{obs}}(T)=\mathbf{z}) \| \boldsymbol{\nu}(\cdot \mid \hat{\mathbf{x}}^{\text{obs}}_{\mathbf{\theta}}(T)=\mathbf{z}))] .
\end{equation}
Under assumptions (i) (iii) (iv) (v) (vi) (vii) (viii), the SDE in Eq \eqref{eq:5u} has a corresponding reverse-time SDE given by
\begin{equation}\label{eq:4u}
\mathrm{d} \mathbf{x}^{\text{obs}}=[f(\mathbf{x}^{\text{obs}}, t)-g(t)^2 \nabla_{\mathbf{x}^{\text{obs}}} \log p_t(\mathbf{x}^{\text{obs}})] \mathrm{d} t+g(t) \mathrm{d} \overline{\mathbf{w}} .
\end{equation}
Since Eq \eqref{eq:4u} is the time reversal of Eq \eqref{eq:5u}, it induces the same path measure $\boldsymbol{\mu}$. As a result, $D_{\mathrm{KL}}(\boldsymbol{\mu}(\cdot \mid \mathbf{x}^{\text{obs}}(T)=\mathbf{z}) \| \boldsymbol{\nu}(\cdot \mid \hat{\mathbf{x}}^{\text{obs}}_{\mathbf{\theta}}(T)=\mathbf{z}))$ can be viewed as the KL divergence between the path measures induced by the following two (reverse-time) SDEs:
$$
\begin{gathered}
\mathrm{d} \mathbf{x}^{\text{obs}}=[f(\mathbf{x}^{\text{obs}}, t)-g(t)^2 \nabla_{\mathbf{x}^{\text{obs}}} \log p_t(\mathbf{x}^{\text{obs}})] \mathrm{d} t+g(t) \mathrm{d} \overline{\mathbf{w}}, \quad \mathbf{x}^{\text{obs}}(T)=\mathbf{x}^{\text{obs}}, \\
\mathrm{d} \hat{\mathbf{x}}^{\text{obs}}=[f(\hat{\mathbf{x}}^{\text{obs}}, t)-g(t)^2 \mathbf{s}_{\boldsymbol{\theta}}(\hat{\mathbf{x}}^{\text{obs}}, t)] \mathrm{d} t+g(t) \mathrm{d} \overline{\mathbf{w}}, \quad \hat{\mathbf{x}}^{\text{obs}}_{\mathbf{\theta}}(T)=\mathbf{x}^{\text{obs}} .
\end{gathered}
$$

Under assumptions (vii) (viii) (ix) (x) (xi), we apply the Girsanov Theorem II [\citep{ksendal1987StochasticDE}, Theorem 8.6.6], together with the martingale property of Itô integrals, which yields
\begin{equation}\label{eq:3u}
 \begin{aligned}
 & D_{\mathrm{KL}}(\boldsymbol{\mu}(\cdot \mid \mathbf{x}^{\text{obs}}(T)=\mathbf{z}) \| \boldsymbol{\nu}(\cdot \mid \hat{\mathbf{x}}^{\text{obs}}_{\mathbf{\theta}}(T)=\mathbf{z})) \\
&= \mathbb{E}_{\boldsymbol{\mu}}[\frac{1}{2} \int_0^T g(t)^2\|\nabla_{\mathbf{x}^{\text{obs}}(t)} \log p_t(\mathbf{x}^{\text{obs}}(t))-\mathbf{s}_{\boldsymbol{\theta}}(\mathbf{x}^{\text{obs}}(t), t)\|_2^2 \mathrm{~d} t] \\
&\leq \frac{1}{2(1-\rho_{\text{max}})} \int_0^T \mathbb{E}_{p_t(\mathbf{x}^{\text{obs}}(t))}[g(t)^2\|\nabla_{\mathbf{x}^{\text{obs}}(t)} \log p_t(\mathbf{x}^{\text{obs}}(t))-\mathbf{s}_{\boldsymbol{\theta}}(\mathbf{x}^{\text{obs}}(t), t) \odot \sqrt{\mathbf{1}-\boldsymbol{\rho}}\|_2^2] \mathrm{d} t \\
&=  \frac{1}{2(1-\rho_{\text{max}})} \int_0^T \mathbb{E}_{p_t(\mathbf{x}^{\text{obs}}(t))}[g(t)^2\|\nabla_{\mathbf{x}^{\text{obs}}(t)} \log p_t(\mathbf{x}^{\text{obs}}(t))-\mathbf{s}_{\boldsymbol{\theta}}(\mathbf{x}^{\text{obs}}(t), t) \odot \mathbf{m}\|_2^2] \mathrm{d} t \\
&=  \frac{1}{1-\rho_{\text{max}}}J_{\mathrm{SM}}(\mathbf{\theta} ; g(\cdot)^2),
 \end{aligned}
\end{equation}
where $\rho_{\text{max}} = \max_{i=1,\ldots,d} \rho_i$ and $1-\rho_{\text{max}}>0$ by assumption. 
Combining Eqs. \eqref{eq:1u}, \eqref{eq:2u} and \eqref{eq:3u}, we have $D_{\mathrm{KL}}(p_0 \| p_{\boldsymbol{\theta}}) \leq \frac{1}{1-\rho_{\text{max}}}J_{\mathrm{SM}}(\boldsymbol{\theta} ; g(\cdot)^2)+D_{\mathrm{KL}}(p_T \| \pi)$, which further yields $-\mathbb{E}_{p(\mathbf{x}^{\text{obs}})}[\log p_{\mathbf{\theta}}(\mathbf{x})] \leq \frac{1}{1-\rho_{\text{max}}}J_{\mathrm{DSM}}(\boldsymbol{\theta} ; g(\cdot)^2)+C_1$ by Lemma \ref{DSM-equ}, where $C_1$ is a constant independent of $\mathbf{\theta}$.

\begin{lemma}
\label{DSM-equ}
Denosing Score Matching on missing data is equivalent to Score Matching on missing data, i.e.,  
\begin{equation}\label{eq:lemma1}
\begin{aligned}
&\mathbb{E}_{p_t(\mathbf{x}^{\text{obs}})}[\|(\mathbf{s}_{\boldsymbol{\theta}}(\mathbf{x}^{\text{obs}}_t, t)-\nabla_{\mathbf{x}^{\text{obs}}} \log p_t(\mathbf{x}^{\text{obs}}_t))\odot \mathbf{m}\|_2^2]   \\ &=\mathbb{E}_{p(\mathbf{x}^{\text{obs}}_0)}  \mathbb{E}_{p(\mathbf{x}^{\text{obs}}_t \mid \mathbf{x}^{\text{obs}}_0)}[\|(\mathbf{s}_{\boldsymbol{\theta}}(\mathbf{x}^{\text{obs}}_t, t)-\nabla_{\mathbf{x}^{\text{obs}}_t} \log p(\mathbf{x}^{\text{obs}}_t \mid \mathbf{x}^{\text{obs}}_0)) \odot \mathbf{m}\|_2^2]+C,
\end{aligned}    
\end{equation}
where $\mathbf{m}=\mathbbm{1}\{\mathbf{x}^{\text{obs}}_0=\mathrm{na}\}$ indicated the missing entries in $\mathbf{x}^{\text{obs}}$ and $C$ is a constant that does not depend on $\boldsymbol{\theta}$. We interchange $\mathbf{x}^{\text{obs}}(t)$ with $\mathbf{x}^{\text{obs}}_t$.
\end{lemma}

\begin{proof}
We begin with the Score Matching on the left-hand side of \eqref{eq:lemma1}
\begin{equation}
\label{eq:ESM}
\begin{aligned}
\text{LHS} &=\mathbb{E}_{p_t(\mathbf{x}^{\text{obs}}_t)}[\|(\mathbf{s}_{\boldsymbol{\theta}}(\mathbf{x}^{\text{obs}}_t, t)-\nabla_{\mathbf{x}^{\text{obs}}_t} \log p_t(\mathbf{x}^{\text{obs}}_t))\odot \mathbf{m}\|_2^2] \\
&=\mathbb{E}_{p_t(\mathbf{x}^{\text{obs}}_t)}[\|\mathbf{s}_{\boldsymbol{\theta}}(\mathbf{x}^{\text{obs}}_t, t)\odot \mathbf{m}\|^2]-S(\theta)+C_2,
\end{aligned}
\end{equation}
where $C_2=\mathbb{E}_{p_t(\mathbf{x}^{\text{obs}}_t)}[\|\nabla_{\mathbf{x}^{\text{obs}}_t} \log p_t(\mathbf{x}^{\text{obs}}_t)\odot \mathbf{m}\|^2]$ is a constant that does not depend on $\boldsymbol{\theta}$, and
$$
\begin{aligned}
S(\theta) & =2\mathbb{E}_{p_t(\mathbf{x}^{\text{obs}}_t)}[\langle\mathbf{s}_{\boldsymbol{\theta}}(\mathbf{x}^{\text{obs}}_t, t), \nabla_{\mathbf{x}^{\text{obs}}_t} \log p_t(\mathbf{x}^{\text{obs}}_t)\odot \mathbf{m}\rangle] \\
& =2\int_{\mathbf{x}^{\text{obs}}_t} p_t(\mathbf{x}^{\text{obs}}_t)\langle\mathbf{s}_{\boldsymbol{\theta}}(\mathbf{x}^{\text{obs}}_t, t), \nabla_{\mathbf{x}^{\text{obs}}_t} \log p_t(\mathbf{x}^{\text{obs}}_t)\odot \mathbf{m}\rangle  \mathrm{~d}\mathbf{x}^{\text{obs}}_t \\
& =2\int_{\mathbf{x}^{\text{obs}}_t} \langle\mathbf{s}_{\boldsymbol{\theta}}(\mathbf{x}^{\text{obs}}_t, t), \nabla_{\mathbf{x}^{\text{obs}}_t} p_t(\mathbf{x}^{\text{obs}}_t)\odot \mathbf{m}\rangle \mathrm{~d}\mathbf{x}^{\text{obs}}_t \\
& =2\int_{\mathbf{x}^{\text{obs}}_t}\langle\mathbf{s}_{\boldsymbol{\theta}}(\mathbf{x}^{\text{obs}}_t, t), \frac{\mathrm{d}}{\mathrm{d} \mathbf{x}^{\text{obs}}_t} \int_{\mathbf{x}^{\text{obs}}_0} p_0(\mathbf{x}^{\text{obs}}_0) p(\mathbf{x}^{\text{obs}}_t \mid \mathbf{x}^{\text{obs}}_0)  \odot \mathbf{m}\mathrm{~d} \mathbf{x}^{\text{obs}}_0 \rangle \mathrm{~d}\mathbf{x}^{\text{obs}}_t \\
& =2\int_{\mathbf{x}^{\text{obs}}_t} \int_{\mathbf{x}^{\text{obs}}_0} p_0(\mathbf{x}^{\text{obs}}_0) p(\mathbf{x}^{\text{obs}}_t \mid \mathbf{x}^{\text{obs}}_0)\langle\mathbf{s}_{\boldsymbol{\theta}}(\mathbf{x}^{\text{obs}}_t, t), \frac{\mathrm{d}\log p(\mathbf{x}^{\text{obs}}_t \mid \mathbf{x}^{\text{obs}}_0) }{\mathrm{d} \mathbf{x}^{\text{obs}}_t} \odot \mathbf{m}  \rangle \mathrm{~d} \mathbf{x}^{\text{obs}}_0 \mathrm{d}\mathbf{x}^{\text{obs}}_t \\
& =2\mathbb{E}_{p(\mathbf{x}^{\text{obs}}_t, \mathbf{x}^{\text{obs}}_0)}[\langle\mathbf{s}_{\boldsymbol{\theta}}(\mathbf{x}^{\text{obs}}_t, t), \frac{\mathrm{d}\log p(\mathbf{x}^{\text{obs}}_t \mid \mathbf{x}^{\text{obs}}_0) }{\mathrm{d} \mathbf{x}^{\text{obs}}_t} \odot \mathbf{m}\rangle] .
\end{aligned}
$$
Substituting this expression for $S(\theta)$ into 
Eq \eqref{eq:ESM} yields
\begin{equation}
\label{eq:ESMv2}
\begin{aligned}
\text{LHS} & = \mathbb{E}_{p_t(\mathbf{x}^{\text{obs}}_t)}[\|\mathbf{s}_{\boldsymbol{\theta}}(\mathbf{x}^{\text{obs}}_t, t)\odot \mathbf{m}\|^2] \\
& -2\mathbb{E}_{p(\mathbf{x}^{\text{obs}}_t, \mathbf{x}^{\text{obs}}_0)}[\langle\mathbf{s}_{\boldsymbol{\theta}}(\mathbf{x}^{\text{obs}}_t, t), \frac{\mathrm{d}\log p(\mathbf{x}^{\text{obs}}_t \mid \mathbf{x}^{\text{obs}}_0) }{\mathrm{d} \mathbf{x}^{\text{obs}}_t} \odot \mathbf{m}\rangle]+C_2 .
\end{aligned}
\end{equation}

On the other hand, we also have the Denoising Score Matching objective on the right-hand side of \eqref{eq:lemma1} is 
\begin{equation}
\label{eq:DSM}
\begin{aligned}
\text{RHS} & =\mathbb{E}_{p_t(\mathbf{x}^{\text{obs}}_t)}[\|\mathbf{s}_{\boldsymbol{\theta}}(\mathbf{x}^{\text{obs}}_t, t)\odot \mathbf{m}\|^2] \\
& -2\mathbb{E}_{p(\mathbf{x}^{\text{obs}}_t, \mathbf{x}^{\text{obs}}_0)}[\langle\mathbf{s}_{\boldsymbol{\theta}}(\mathbf{x}^{\text{obs}}_t, t), \frac{\mathrm{d}\log p_t(\mathbf{x}^{\text{obs}}_t \mid \mathbf{x}^{\text{obs}}_0) }{\mathrm{d} \mathbf{x}^{\text{obs}}_t}\rangle\odot \mathbf{m}]+C_3,
\end{aligned}
\end{equation}
where $C_3=\mathbb{E}_{p(\mathbf{x}^{\text{obs}}_t, \mathbf{x}^{\text{obs}}_0)}[\|\frac{\mathrm{d}\log p_t(\mathbf{x}^{\text{obs}}_t \mid \mathbf{x}^{\text{obs}}_0) }{\mathrm{d} \mathbf{x}^{\text{obs}}_t} \odot \mathbf{m}\|^2]+C$ is a constant that does not depend on $\boldsymbol{\theta}$.

Comparing equations \eqref{eq:ESMv2} and \eqref{eq:DSM}, we thus show that the two optimization objectives are equivalent up to a constant.   
\end{proof}

\section{More Details on Experiments}
\subsection{Datasets}\label{sec:bayesdata}
\paragraph{Details of the Bayesian Network} \label{ap:BN}
Figure \ref{fig:BN} demonstrates the Bayesian Network for generating the tabular data. It contains two continuous variables C1, C2, and three discrete random variables D1, D2, and D3. 
The distribution of these variables is set as follows. The marginal distribution of C1 is $\mathcal{N}(25, 2)$, the conditional distribition of C2 given C1 is $\text{C2}|\text{C1} \sim \mathcal{N}(0.1 \cdot \text{C1} +50 , 5)$, and the marginal distribution of D1 is $Bernoulli(0.3)$, where $Bernoulli(\xi)$ stands for the Bernoulli distribution with mean equal to $\xi$.
The conditional distribution of D2, given C1, C2 and D1, is set as 
\[ \text{D2}|\text{C1}, \text{C2}, \text{D1} \sim  \begin{cases} 
      Ca(0.3,0.6,0.1) & \text{C1}>26 , \text{C2}>55 , \text{D1}=1; \\
      Ca(0.2,0.3,0.5) & \text{C1}>26 , \text{C2}\leq55 , \text{D1}=1; \\
      Ca(0.7,0.1,0.2) & \text{C1}\leq26 , \text{C2}>55 , \text{D1}=1; \\ 
      Ca(0.1,0.2,0.7) & \text{C1}\leq26 , \text{C2}\leq55 , \text{D1}=1; \\
      Ca(0.05,0.05,0.9) &  \text{D1}=0, \\ 
   \end{cases}
\] where $Ca(p1,p2, 1-p1-p2)$ denotes the categorical (discrete) distribution for three pre-specified categories. The conditional distribution of D3 given D2 is 
\[  \text{D3}|\text{D2} \sim \begin{cases} 
      Bernoulli(0.2) & \text{D2}=0; \\
      Bernoulli(0.4) & \text{D2}=1; \\
      Bernoulli(0.8) & \text{D2}=2.
   \end{cases}
\]

\begin{figure}
\centering
  \includegraphics[scale=0.9]{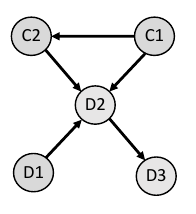}
  \caption{The demonstration of the Bayesian Network for generating the tabular data. ``C1'' and ``C2'' denote the continuous variables and ``D1'', ``D2'', ``D3'' denotes the discrete random variables. The marginal/conditional distributions for each node are detailed in Section \ref{sec:bayesdata}.}
  \label{fig:BN}
\end{figure}

\subsection{Implementation Details}\label{ap:training}

We adopt four layers residual network as the backbone of the diffusion model. The dimension of the diffusion embedding is 128 with channels as 64. We set the minimum noise level $\beta_1=0.0001$ and the maximum noise level $\beta_T=0.5$ in Algorithm \ref{alg} and Algorithm \ref{alg1} with quadratic schedule
$$
\beta_t=\left(\frac{T-t}{T-1} \sqrt{\beta_1}+\frac{t-1}{T-1} \sqrt{\beta_T}\right)^2.
$$ 
We mainly follow the hyperparameter in the previous works that train the diffusion model on tabular data \cite{Tashiro2021CSDICS,Zheng2022DiffusionMF}. 
We use the Adam optimizer with MultiStepLR with 0.1 decay at $25 \%, 50 \%, 75 \%$, and $90 \%$ of the total epochs and with an initial learning rate as 0.0005. 

With regard to the baselines of STaSy, we adopt the same setting of its open resource implementation \footnote{\url{https://openreview.net/forum?id=1mNssCWt_v}}, i.e., Varaince Exploding SDE with six layers ConcatSquash network as the backbone of the diffusion model and Fourier embedding, the adam optimizer with learning rate as 2e-03, training with batch size 64 and 250 epochs/1000 epochs with additional 50 finetuning epochs.

For the downstream classifier/regressor, we adopt the same base hyperparameters in [\cite{Kim2023STaSyST}, Table 26].

\subsection{Additional Experiental Results}\label{ap:exper}

\subsubsection{Additional Results for Other Criteria for {\it Utility} Evaluation}
Table \ref{tab:t2d}, \ref{tab:t4d}, and \ref{tab:t6d} provide the additional experimental results for other criteria under {\it Utility} evaluation for Table \ref{tab:t2}, \ref{tab:t4}, and \ref{tab:t6} in the main paper, i.e., the F1, Weighted-F1, AUROC for the classification task and $R^2$ for the regression task. A detailed explanation of the above-mentioned criteria can be found in \cite{Kim2023STaSyST}. To make our paper self-contained, we briefly restate it here.
\begin{enumerate}
    \item Binary F1 for binary classification: sklearn.metrics.f1\_score with `average'=`binary'.
    \item Macro F1 for multi-class classification: sklearn.metrics.f1\_score with `average'=`macro'.
    \item Weighted-F1: $=\sum_{i=0}^K w_i s_i$, where $K$ denotes the number of classes, the weight of $i$-th class $w_i$ is $\frac{1-p_i}{K-1}, p_i$ is the proportion of $i$-th class's cardinality in the whole  dataset, and score $s_i$ is a per-class F1 of $i$-th class (in a One-vs-Rest manner).
    \item AUROC: sklearn.metrics.roc\_auc\_score.
\end{enumerate}
From the results in Table \ref{tab:t2d}, \ref{tab:t4d}, and \ref{tab:t6d}, it can be seen that the proposed \textit{MissDiff} consistently outperforms the compared methods in most instances. For the column missing case, \textit{MissDiff} tends to perform worse, which indicates the potential limitations of the proposed method for future investigations.

\begin{table}[htbp]
\centering

 \caption{{\it Utility} evaluation of {\it MissDiff} on Census dataset with other criteria. ``-'' denotes the corresponding method cannot applied since no data $\mathbf{x}_i$ will be left after deleting the incomplete data.}
  \resizebox{\textwidth}{!}{
\begin{tabular}{c|c|c|c|c|c|c}
\toprule  Criterian & Missing Mechanism & {\it MissDiff} & {\it Diff-delete}  & {\it Diff-mean}  & {\it STaSy-delete}  & {\it STaSy-mean} \\
\midrule
\multirow{3}[0]{*}{Binary F1}&Row Missing                        & \textbf{0.344}  &    -     & 0.280   & -  &  0.314 \\
&Column  Missing                    & 0.141   & 0.063 & 0.413  & \textbf{0.509} & 0.383\\
&Independent Missing                & \textbf{0.291}   & 0.045 & 0.225 & 0.274 & 0.241\\ 
\midrule
\multirow{3}[0]{*}{Weighted-F1}&Row Missing                        & \textbf{0.470}   &    -     & 0.423   & -  &  0.488 \\
&Column  Missing                    & 0.305   & 0.249 & 0.523  & \textbf{0.571} & 0.490\\
&Independent Missing                & \textbf{0.431}   & 0.237 & 0.375 & 0.416 & 0.389\\ 
\midrule
\multirow{3}[0]{*}{AUROC}&Row Missing                        & \textbf{0.772}   &    -     & 0.685   & -  &  0.731 \\
&Column  Missing                    & 0.539   & 0.469 & \textbf{0.757}  & 0.750 & 0.637\\
&Independent Missing                & \textbf{0.650}   & 0.474 & 0.655 & 0.621 & 0.613\\ 

\bottomrule
\end{tabular}}
\label{tab:t2d}
\end{table}

\begin{table}[ht!]
\centering
 \caption{{\it Utility} evaluation of {\it MissDiff} on MIMIC4ED dataset with $R^2$ criterion.}
\begin{tabular}{c|c|c|c|c|c}
\toprule Missing mechanism  & {\it MissDiff} & {\it Diff-delete}  & {\it Diff-mean}  & {\it STaSy-delete}  & {\it STaSy-mean}  \\
\midrule
Row Missing                          & \textbf{0.088}   &   -      & 0.057  &   -    & 0.067 \\
Rolumn Missing                     & \textbf{0.095}   & - & 0.023  &   -      &  0.073\\
Independent Missing                & \textbf{0.156}   & - & 0.062 &   -      &  0.142 \\ 

\bottomrule
\end{tabular}
\label{tab:t4d}
\end{table}

\begin{table}[ht!]
\centering
 \caption{{\it Utility} evaluation of {\it MissDiff} on Census dataset under MAR, NMAR with other criteria.}
\begin{tabular}{c|c|c|c|c}
\toprule  Criterian & Missing Mechanism  & {\it MissDiff} & {\it Diff-delete}  & {\it Diff-mean}    \\
\midrule
\multirow{2}[0]{*}{Binary F1}& MAR                         & \textbf{0.346}	& 0.108	& 0.224   \\
& NMAR                     & \textbf{0.464}    &	0.233    &	0.383   \\
\midrule
\multirow{2}[0]{*}{Weighted-F1}& MAR                         & \textbf{0.473}	& 0.276	& 0.376   \\
& NMAR                     & \textbf{0.564}    &	0.364    &	0.501   \\
\midrule
\multirow{2}[0]{*}{AUROC}& MAR                         & \textbf{0.833}	& 0.441	& 0.774   \\
& NMAR                     & \textbf{0.834}    &	0.499    &	0.746   \\
\bottomrule
\end{tabular}
\label{tab:t6d}
\end{table}

\subsubsection{Experiment Results for Different Classifiers/Regressors}

As mentioned in section \ref{sec:evaluation}, we train various models, including Decision Tree, AdaBoost, Logistic/Linear Regression, MLP classifier/regressor, RandomForest, and XGBoost, on synthetic data. Table \ref{tab:t9} to \ref{tab:t13} present the corresponding results on different classifiers/regressors, from which we can see that \textit{MissDiff} still performs well under most cases.

\begin{table}[htbp]
\centering
 \caption{{\it Utility} evaluation of {\it MissDiff} on Census dataset by Decision Tree.} 
\begin{tabular}{c|c|c|c|c|c}
\toprule   & {\it MissDiff} & {\it Diff-delete}  & {\it Diff-mean}  & {\it STaSy-delete}  & {\it STaSy-mean} \\
\midrule
Row Missing                       & \textbf{78.08}\%   &    -     & 74.55\%   & -  &  60.74\% \\
Column  Missing                    & 62.65\%   & 69.10\% & \textbf{78.88}\%  & 65.38\% & 66.31\%\\
independent                 & \textbf{80.68}\%   & 72.68\% & 67.70\% & 76.35\% & 55.99\%\\ 

\bottomrule
\end{tabular}
\label{tab:t9}
\end{table}

\begin{table}[htbp]
\centering
 \caption{{\it Utility} evaluation of {\it MissDiff} on Census dataset by AdaBoost.} 
\begin{tabular}{c|c|c|c|c|c}
\toprule   & {\it MissDiff} & {\it Diff-delete}  & {\it Diff-mean}  & {\it STaSy-delete}  & {\it STaSy-mean} \\
\midrule
Row Missing                       & \textbf{80.38}\%   &    -     & 79.28\%   & -  &  73.23\% \\
Column  Missing                    & 72.18\%   & 76.30\% & \textbf{80.65}\%  & 69.60\% & 42.24\%\\
independent                 & \textbf{78.70}\%   & 76.13\% & 75.96\% & 76.55\% & 78.39\%\\ 

\bottomrule
\end{tabular}
\label{tab:t10}
\end{table}

\begin{table}[htbp]
\centering
 \caption{{\it Utility} evaluation of {\it MissDiff} on Census dataset by Logistic Regression.} 
\begin{tabular}{c|c|c|c|c|c}
\toprule   & {\it MissDiff} & {\it Diff-delete}  & {\it Diff-mean}  & {\it STaSy-delete}  & {\it STaSy-mean} \\
\midrule
Row Missing                       & \textbf{79.20}\%   &    -     & 77.08\%   & -  &  71.04\% \\
Column  Missing                    & 73.50\%   & 76.30\% & \textbf{77.45}\%  & 66.91\% & 69.08\%\\
independent                 & 76.20\%   & \textbf{76.30}\% & 76.25\% & 77.13\% & 69.68\%\\ 

\bottomrule
\end{tabular}
\label{tab:t11}
\end{table}

\begin{table}[H]
\centering
 \caption{{\it Utility} evaluation of {\it MissDiff} on Census dataset by Multi-layer Perceptron (MLP).} 
\begin{tabular}{c|c|c|c|c|c}
\toprule   & {\it MissDiff} & {\it Diff-delete}  & {\it Diff-mean}  & {\it STaSy-delete}  & {\it STaSy-mean} \\
\midrule
Row Missing                       & \textbf{77.70}\%   &    -     & 75.13\%   & -  &  49.78\% \\
Column  Missing                    & 68.33\%   & 65.75\% & \textbf{75.00}\%  & 70.97\% & 58.83\%\\
independent                 & \textbf{75.33}\%   & 72.18\% & 74.30\% & 76.81\% & 37.59\%\\ 

\bottomrule
\end{tabular}
\label{tab:t12}
\end{table}

\begin{table}[H]
\centering
 \caption{{\it Utility} evaluation of {\it MissDiff} on Census dataset by Random Forest. } 
\begin{tabular}{c|c|c|c|c|c}
\toprule   & {\it MissDiff} & {\it Diff-delete}  & {\it Diff-mean}  & {\it STaSy-delete}  & {\it STaSy-mean} \\
\midrule
Row Missing                       & \textbf{80.10}\%   &    -     & 77.13\%   & -  &  72.68\% \\
Column  Missing                    & 73.68\%   & 76.33\% & \textbf{79.88}\%  & 74.70\% & 71.58\%\\
independent                 & \textbf{79.33}\%   & 76.30\% & 76.38\% & 76.31\% & 76.98\%\\ 

\bottomrule
\end{tabular}
\label{tab:t13}
\end{table}

\subsubsection{Additional Results for {\it STaSy-delete} and {\it STaSy-mean}}

In section \ref{sec:expri}, we mentioned if we train {\it STaSy-delete} and {\it STaSy-mean} as the same training epochs (250 epochs) on the Census dataset under MCAR as {\it MissDiff}, their performance is significantly worse, which are demonstrated in Table \ref{tab:t7} and \ref{tab:t8}. This observation highlights that the proposed \textit{MissDiff} requires considerably fewer training epochs compared to STaSy in order to achieve satisfactory results when handling data with missing values. 

\begin{table}[H]
\centering
 \caption{{\it Fidelity} evaluation of {\it MissDiff} on Census dataset with 250 training epochs.}
\begin{tabular}{c|c|c|c|c|c}
\toprule
 & {\it MissDiff} & {\it Diff-delete}  & {\it Diff-mean} & {\it STaSy-delete}  & {\it STaSy-mean} \\
\midrule
Row Missing                       & \textbf{80.59}\%   &     -    & 76.92\%  & -& 50.08\% \\
Column  Missing                    & \textbf{82.70}\%   & 75.03\% & 76.17\%  & 52.49\% & 49.63\% \\
independent                 & \textbf{83.16}\%   & 74.94\% & 76.60\% & 53.7\% & 50.11\% \\

\bottomrule
\end{tabular}
\label{tab:t7}
\end{table}

\begin{table}[H]
\centering
 \caption{{\it Utility} evaluation of {\it MissDiff} on Census dataset with 250 training epochs.} 
\begin{tabular}{c|c|c|c|c|c}
\toprule   & {\it MissDiff} & {\it Diff-delete}  & {\it Diff-mean}  & {\it STaSy-delete}  & {\it STaSy-mean} \\
\midrule
Row Missing                       & \textbf{79.48}\%   &    -     & 78.45\%   & -  &  60.96\% \\
Column  Missing                    & 71.68\%   & 72.89\% & \textbf{79.60}\%  & 56.19\% & 61.46\%\\
independent                 & \textbf{79.49}\%   & 75.39\% & 75.96\% & 49.78\% & 70.68\%\\ 

\bottomrule
\end{tabular}
\label{tab:t8}
\end{table}

\subsection{Computational Time}
All the experiments are conducted on NVIDIA A100 Tensor Core GPUs. It takes around 30 minutes for each experiment on Bayesian Network, around 5 hours for each experiment on the Census dataset, and around one day for each experiment on the MIMIC4ED dataset. 

\end{document}